\documentclass{article}
\PassOptionsToPackage{numbers,compress,sort}{natbib}
\PassOptionsToPackage{pagebackref}{hyperref}

\usepackage[preprint]{neurips_2022}

\usepackage[utf8]{inputenc} 
\usepackage[T1]{fontenc}    
\usepackage{hyperref}       
\usepackage{url}            
\usepackage{booktabs}       
\usepackage{amsfonts}       
\usepackage{nicefrac}       
\usepackage{microtype}      
\usepackage{xcolor}         
\usepackage{amsmath}
\usepackage{amssymb}
\usepackage{mathtools}
\usepackage{amsthm}
\usepackage{caption}
\usepackage{graphicx}
\usepackage{subcaption}
\usepackage{wrapfig}
\newcommand\norm[1]{\lVert#1\rVert}
\theoremstyle{plain}
\newtheorem{theorem}{Theorem}[section]

\newtheorem{lemma}[theorem]{Lemma}
\newtheorem{corollary}[theorem]{Corollary}
\theoremstyle{definition}

\theoremstyle{remark}
\newtheorem{remark}[theorem]{Remark}

\title{A Flexible Diffusion Model}

\author{%
  Weitao Du \\
  Chinese Academy of Sciences\\
  \texttt{duweitao@amss.ac.cn}
\\
\And
  Tao Yang \\
  Xi’an Jiaotong University\\
  \texttt{ yt14212@stu.xjtu.edu.cn}
\\
\And
  He Zhang \\
  Xi’an Jiaotong University\\
  \texttt{ mao736488798@stu.xjtu.edu.cn}
\And
  Yuanqi Du \\
  Cornell University\\
  \texttt{yd392@cornell.edu} \\
}

\begin{document}

\maketitle
\begin{abstract}
Diffusion (score-based) generative models have been widely used for modeling various types of complex data, including images, audios, and point clouds. Recently, the deep connection between forward-backward stochastic differential equations (SDEs) and diffusion-based models has been revealed, and several new variants of SDEs are proposed (e.g., sub-VP, critically-damped Langevin) along this line. Despite the empirical success of the hand-crafted fixed forward SDEs, a great quantity of proper forward SDEs remain unexplored. In this work, we propose a general framework for parameterizing the diffusion model, especially the spatial part of the forward SDE. An abstract formalism is introduced with theoretical guarantees, and its connection with previous diffusion models is leveraged. We demonstrate the theoretical advantage of our method from an optimization perspective.
Numerical experiments on synthetic datasets, MINIST and CIFAR10 are also presented to validate the effectiveness of our framework.
\end{abstract}
\section{Introduction}
Diffusion (score-based) models, which are originated from non-equilibrium statistical physics, have recently shown impressive  successes on sample generations of a wide range of types, including images \citep{ho2020denoising,nichol2021improved,song2020score,dhariwal2021diffusion} , 3D point clouds \citep{luo2021diffusion,2110.14811} and audio generation \citep{kong2020diffwave,liu2021diffsvc}, among others. In addition to concrete applications of various diffusion generative models, it is also desirable to analyze them in an appropriate and flexible framework, by which novel improvements can be further developed.    

Currently, one of the promising formal frameworks for unifying different types of diffusion models is the stochastic differential equations (SDE), as proposed in \citep{song2020score}. Under this formalism, the generative (denoising) process can be viewed as reversing the forward (noising) process from real data manifold to white noise. Furthermore, with the help of the Feynman-Kac formula and Girsanov transform \citep{da2014introduction}, the score-matching training has been proved to be equivalent to certain log-likelihood training in the infinite-dimensional path space \citep{huang2021variational}. 

Although the ELBO cost function for diffusion models derived in \citep{huang2021variational} explicitly contains both the forward and backward ingredients, the forward (noising) process is hand-crafted and set to be fixed throughout training, which is clearly a mismatch when comparing diffusion models with other log-likelihood based models (e.g., VAE \citep{oussidi2018deep}) in practice. Moreover, since the reverse generating process is uniquely determined by the forward process, the total flexibility of the model lies in parameterizing the forward process. This would to another cascading effect of freezing the forward process, especially when knowing that different noising schedules do affect the empirical performances (e.g., the different forward processes including VE, VP, sub-VP \citep{song2020score} and damped Langevin diffusion \citep{dockhorn2022score} displayed distinct generating performances in terms of perceptual quality).   
On the other hand, even reparameterizing and training the noising-schedule of the forward process would improve the diffusion model, as it was shown in \citep{kingma2021variational}. 

To answer these concerns, it is crucial to incorporate flexible parameterized forward processes under the general SDE framework of \citep{song2020score}. Though the idea of training the forward process is intuitively reasonable, the implementation is far from straightforward. The first challenge is to find the appropriate search sub-space within the grand function space consisting of the whole stochastic processes. A hard constraint is that the stationary distribution of the candidate stochastic processes is centered Gaussian, which will be set as the prior distribution of the generative process to sample from.  
In fact, \cite{kingma2021variational} has demonstrated the power of optimizing the one-dimensional time component (reparameterized by the signal to noise ratio)  of the forward process. However, efficiently parameterizing the space components remains to be explored, especially taking into account the complex structure of the data distribution \citep{narayanan2010sample}.     

This paper concentrates on the theoretical and the practical aspects of solving the flexibility challenges of the diffusion model, emphasizing the spatial components of the forward process. First of all, inspired by concepts from Riemannian geometry and  Hamilton Monte-Carlo methods, we define a unified and flexible class of diffusion processes (FP-Diffusion) that rigorously satisfies the fixed Gaussian stationary distribution condition with theoretical guarantee. 
To highlight the advantages of flexible diffusion models, we discuss the theoretical properties and effects of parameterized forward processes from the perspective of optimization.
Furthermore, by introducing the flexible diffusion model, all sorts of regularizers \citep{finlay2020train,onken2021ot} on selecting more flat paths for continuous normalizing flows can be referenced and implemented.  

We summarize our major contributions here:
\begin{itemize}
    \item We introduce the symplectic structure and the anisotropic Riemannian structure to the diffusion model, through which a framework for parameterizing the forward process with theoretical guarantees is formulated. Completeness and convergence properties as $t \rightarrow \infty$ are proved along the same route.  
    \item To theoretically motivate optimizing the forward process, we analyze the implications of parameterizing the forward (noising) process from the optimization function point of view and demonstrate how our method unifies previous diffusion models. Since our extension is compatible with merging regularization terms into the training loss, We provide experimental results in simulated scenarios and demonstrate how it behaves under regularization. 
    \item Except considering the continuous diffusion framework, we also develop a corresponding light version of our parameterization with an explicit formula of general terms for efficient Monte-Carlo training. It enables us to perform comparative studies on large-scale datasets, e.g., MNIST and CIFAR10.
\end{itemize}

\begin{figure}[t]
\centering
\includegraphics[width=0.8\textwidth]{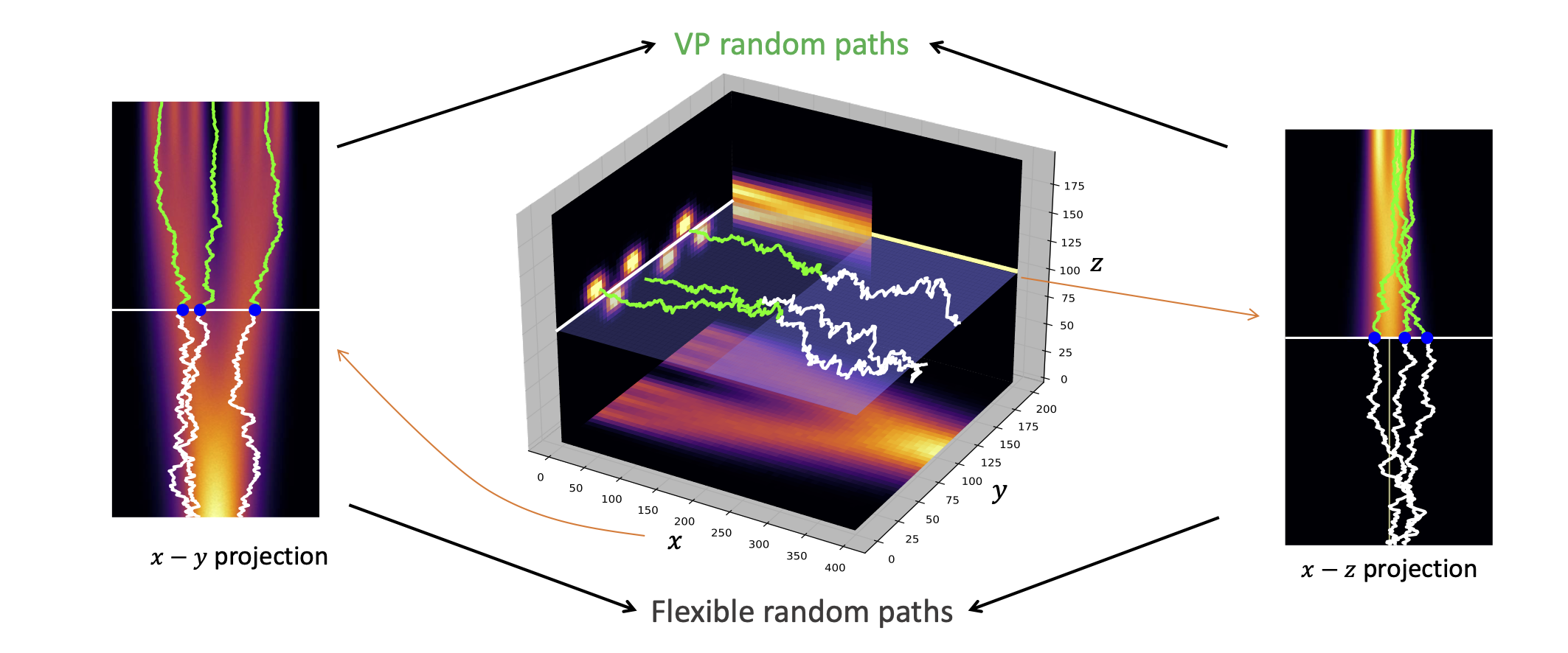}
\caption{Evolution trajectories of fixed and flexible forward SDEs}
\label{fig:motivation}
\vspace{-0.3cm}
\end{figure}


\section{Preliminaries and Related Works}
Given a data distribution $p(x)$, we associate it with a Gaussian diffusion process (forward) that increasingly adds noise to the data, then the high-level idea of diffusion generative models is to model the real data distribution as a multi-step denoising (backward) process. In a discrete step setting, the forward process can be formulated as an N-steps Markov chain from real data $x$ to each noised $x_t$:
$$p(x_t| x_{t-1}) = \mathcal{N}(\alpha_t x, \beta_t I),\ \ t \in \{1,\dots ,N\}. $$
For DDPM model \citep{ho2020denoising}, $\alpha_t$ is set to be $\alpha_t : = \sqrt{1 - \beta_t}$. If we take the continuous limit of $\beta_t$ (when $\sqrt{1 - \beta_t} \approx 1 - \frac{1}{2}\beta_t $), 
we find that $X_t$ satisfies a time-change of the following Ornstein–Uhlenbeck stochastic differential equation (SDE):
\begin{equation} \label{ou}
dX_t = -\frac{1}{2}\beta(t) X_tdt + \sqrt{ \beta(t)}dW_t,
\end{equation}
the so-called \emph{variance-preserving} diffusion process (VP) in \citep{song2020score}. Therefore, DDPM can be taken as a discretization of the Ornstein–Uhlenbeck process. Following this line,  \citep{song2020score} proposed to characterize different types of diffusion models by formulating the underlying SDE of the model:
\begin{equation} \label{forward}
dX_t = f(X_t,t)dt + g(t)dW_t, \ \ 0 \leq t \leq T    
\end{equation}   
where $\{W_t\}_{t=0}^{\infty}$ denotes the standard Brownian motion, whose dimension is set to be the same as the data. Usually we choose a different time schedule (time-change) from $t$. Let $\beta(t)$ be a continuous function of time $t$ such that $\beta(t) > \beta(s) > 0$ for $0 < s < t$, then $\beta(t)$ is called a time schedule (time-change) of $t$. It can be further shown that when $t \rightarrow \infty$, the stationary distribution of Eq. \ref{ou} is the standard multivariate Gaussian: $\mathcal{N}(0,I)$ \citep{2002Stochastic}. On the other hand, SMLD diffusion models \citep{song2019generative} can be seen as a discretization of the \emph{variance-exploring} (VE) process ((9) of \cite{song2020score}) $\{X_t\}_{t=0}^{t = T}$, which satisfies the following SDE:
\begin{equation} \label{VE}
dX_t = \sqrt{2\sigma(t)\sigma'(t)}dW_t.    
\end{equation}

Note that although a random process doesn't have a pointwise inverse, it's still valid to define a reverse SDE $Y_t$ of a forward SDE $X_t$ such that the marginal distributions at each time and its corresponding `reverse' time match: $p_t(X_t) \equiv q_{T-t}(Y_{T-t})$. Then $Y_t$ is exactly the denoising (backward) stochastic process we are looking for. In other words, real data can be generated by sampling from the Gaussian distribution and tracking the denoising process from time $T$ to $0$. Surprisingly, the solution of the reverse-time SDE with respect to the forward $X_t$ is derived analytically in \citep{anderson1982reverse,song2020score}:
\begin{equation} \label{rev}
dY_t = [f(Y_t,t) - g^2(Y_t,t)\nabla \log p_t(Y_t)]dt + g(t)dW_t,   \end{equation}
where $W_t$ is a Brownian motion running backwards in time from $T$ to 0. The unknown score function $s_t(x) : = \nabla \log p_t(x)$ depends on the data distribution $p_0$ and the forward process $X_t$. The continuous diffusion model approaches the score function by various types of (weighted) score-matching procedures, we will briefly review the loss function in section \ref{main:param_forward}. 

Now we summarize more related works: 

\textbf{Diffusion Probabilistic Model}
(DPM) as a generative model~\citep{VAE, GAN, yang2021towards, DisCo} was first introduced in ~\cite{sohl2015deep}, as a probabilistic model inspired by non-equilibrium thermodynamics. The high-level idea is to treat the data distribution as the Gibbs (Boltzmann) equilibrium distribution \cite{friedli_velenik_2017}, then the generating process corresponds to transitioning from non-equilibrium to equilibrium states \citep{de2013non}. DDPM~\cite{ho2020denoising} and ~\cite{nichol2021improved,song2020denoising,  Watson2022LearningFS,Jolicoeu,bao2022analytic} further improve DPMs by introducing the Gaussian Markov chain and various inference and sampling methods, through which the generative model is equivalent to a denoising diffusion model. 
~\citep{vahdat2021score} then introduces the diffusion to a latent space and the denoising steps are also further increased with improved empirical performance. Furthermore, as we will show in this article, there are infinite processes (thermodynamical systems) that can connect non-equilibrium states to equilibrium. 

\textbf{Score Matching}: Score-based energy model \citep{DBLP:journals/jmlr/Hyvarinen05,Vincent2011ACB} is based on minimizing the difference between the derivatives of the data and the model’s log-density
functions, which avoids fitting the normalization constant of the intractable distribution.  ~\citep{song2019generative,song2020sliced} introduce sliced score matching that enabled scalable generative training by leveraging different levels of Gaussian noise and several empirical tricks. ~\cite{song2020score,song2021maximum} further studied perturbing the data with a continuous-time stochastic process.
Under this framework, \citep{kingma2021variational} proposes to reparameterize the time variable of the forward process by the signal-to-noise ratio (SNR) variable and train the noising-schedule. Indeed, the sub-VP SDE of \citep{song2020score} can also be seen as modifying the time scale of the diffusion part of the original VP. From this point of view, our model can be seen as a novel spatial parameterization of the forward process, which takes into account the spatial inhomogeneity of the data.

\section{Methods}
\subsection{A general framework for parameterizing diffusion models}
\label{para}


From the preliminary section, we realize that the stationary distribution of the forward process corresponds to the initial distribution of the denoising (generative) process. Therefore, it must be a simple distribution we know how to sample from, mainly set to be standard Gaussian. 
In this article, we parameterize the spatial components of the forward process by considering the following SDE:
\begin{equation} \label{ge}
dX_t = f(X_t)dt + \sqrt{2R(X_t)}dW_t,  
\end{equation}
under the \textbf{hard  constraint} that the stationary distribution of $X_t$ is standard Gaussian (the scaled Gaussian case can be easily generalized). Introducing the time change $\beta(t)$, then by Ito's formula, $X_{\beta(t)}$ satisfies a variant of Eq. \ref{ge}:
\begin{equation} \label{ge variant}
dX_{\beta(t)} = f(X_t)\beta'(t)dt + \sqrt{2\beta'(t) R(X_t)}dW_t.  
\end{equation}

To satisfy the theoretical constraint, it's obvious that the function class of $f(x)$ and $R(x)$ should be properly restricted. Inspired by ideas from the  Riemannian Manifold Hamiltonian Monte-Carlo algorithm \citep{girolami2011riemann,betancourt2013general,betancourt2017conceptual,NIPS2014_beed1360} and anisotropic diffusion technique of image processing, graph deep learning \citep{weickert1998anisotropic,Perona1990,Alvarez,10.1145/3507905}, we propose a flexible framework for parameterizing the forward process by introducing two geometric quantities: the Riemannian metric and the symplectic form in $\mathbf{R}^n$.

Intuitively, an anisotropic Riemannian metric implies that the space was curved and in-homogeneous, and the corresponding Brownian motion will inject non-uniform noise along with different directions. On the other hand, the symplectic form is crucial for defining the dynamics of a given Hamiltonian. Both of them set the stage for performing diffusion on the data manifold, from the data distribution to the standard multivariate normal distribution, whose density under the canonical volume form $d x_1 \dots d x_n$ is
\begin{equation} \label{gau}
\frac{1}{\sqrt{(2\pi)^n}}\exp(-\frac{1}{2}\norm{x}^2)d x_1 \dots d x_n.
\end{equation}
Moreover, 
in Appendix A,
we will prove that our parameterization is \textbf{complete} under certain constraints.

Now we introduce the two concepts in detail. In a coordinate system, a Riemannian metric can be identified as a symmetric positive-definite matrix ( the Euclidean metric is exactly the identity matrix). Given a Riemannian metric $R(x) : = \{R_{ij}(x)\}_{1 \leq i,j \leq n}$, recall that for a smooth function $H(x)$, the Riemannian Langevin process satisfies the following SDE:
\begin{equation}
dX_t = - \Tilde{\nabla H}(X_t) dt + \sqrt{2}dB_t,    
\end{equation}
where $\Tilde{\nabla} H(x) : = R^{-1}(x) \nabla H(x)$ is the gradient vector field of $H$, and $B_t$ denotes the \textbf{Riemannian Brownian motion} \citep{2002Stochastic}. In local coordinates, we have (see (13) of \citep{girolami2011riemann}):
\begin{align*}
dB^i_t = |R(X_t)|^{-1/2} \sum_{j=1}^n \frac{\partial}{\partial x_j}(R^{-1}_{ij}(X_t) |R(X_t)|^{1/2})dt + \sqrt{R^{-1}(X_t)}dW_t^i,\end{align*} for $i \in \{1,2,\dots,n\}$.
One crucial property of the Riemannian Langevin process \citep{2014Analysis} is that the stationary distribution $p(x)$ has the following form:
$$p(x) \propto e^{-H(x)} dV(x),$$
where $dV(x) := \sqrt{|R(x)|}d x_1 \dots d x_n$ is the Riemannian volume form. Transforming back to the canonical volume form and take $H(x) = \frac{1}{4}\norm{x}^2 \cdot \log (|R(x)|)$, we have proved the following lemma:
\begin{lemma}
The stationary distribution of the SDE (Eq. \ref{metric})  below is the standard Gaussian of $\mathbf{R}^n$:
\begin{align}  \label{metric} 
dX_t = \frac{1}{2}[-\sum_j R^{-1}_{ij}(X_t)\cdot(X_t)_j
+ \sum_j \frac{\partial}{\partial x_j} R^{-1}_{ij}(X_t)] dt + \sqrt{R^{-1}(X_t)}dW_t .  
\end{align}
\end{lemma}
\begin{remark}
It's worth mentioning that the infinitesimal generator of the Riemannian motion is the Laplacian operator $\Delta$. When acting on a smooth function $f$,
$$\Delta f : = \frac{1}{\sqrt{|R(x)|}}\partial_i(\sqrt{|R(x)|}(R^{-1})_{ij}\partial_j f).$$
It has the same form as the anisotropic diffusion of (1.27) in \citep{weickert1998anisotropic}. The effectiveness of anisotropic noise is explored in section \ref{toy}. 
\end{remark}
On the other hand, introducing a symplectic form $\omega$ allows us to do Hamiltonian dynamics in an even-dimensional space $\mathbf{R}^n$, when $n = 2d$. Since a symplectic form is a non-degenerate closed 2-form, it automatically becomes zero in odd-dimensional spaces. In this article, we will restrict ourselves to a special type of symplectic form, which consists of constant anti-symmetric matrices $\{\omega_{ij}\}_{1 \leq i,j \leq 2d}$. Then the corresponding Hamiltonian dynamics of $H(x)$ is:
\begin{equation} \label{sym}
dX_t = \omega \nabla H(X_t) dt.
\end{equation}

We mainly focus on two remarkable properties of Hamiltonian dynamics: 1. It preserves the canonical volume form (the determinant of the Jacobi matrix equals one); 2. The Hamiltonian function $H(x)$ takes a constant value along the integral curves (see the remark in Appendix A). Using the change of variables formula, we conclude that the probabilistic density of $X_t$ preserves the equilibrium Gibbs distribution:
$$p(x) \propto e^{-H(x)} d x_1 \dots d x_n,$$
when $X_0$ is sampled from the Gibbs distribution. 

Let $H(x) = \frac{1}{2}x^2$, the potential energy of the Harmonic oscillator. Then by merging the Riemannian part (Eq. \ref{metric}) and the symplectic part (Eq. \ref{sym}) 
we obtain the following theorem:
\begin{theorem} \label{general}Suppose $\omega$ is a anti-symmetric matrix, and $R^{-1}(x)$ is a positive-definite symmetric matrix-valued function of $x \in \mathbf{R}^n$.
Then the (unique) stationary distribution of (Eq. \ref{sum})  below is the standard Gaussian (Eq. \ref{gau}) of $\mathbf{R}^n$:
\begin{equation} \label{sum}
 dX_t = \frac{1}{2}[-\sum_j R^{-1}_{ij}(X_t)\cdot(X_t)_j - 2\sum_j  \omega_{ij}\cdot(X_t)_j + \sum_j \frac{\partial}{\partial x_j} R^{-1}_{ij}(X_t)] dt  + \sqrt{R^{-1}(X_t)}dW_t   
\end{equation}
\end{theorem}

We name Eq. \ref{sum} as our \textbf{FP-Diffusion} model. Note that theorem \ref{general} can also be verified and extended to scaled Gaussian distribution by direct computation, and we leave the proof in Appendix A.
For a graphical presentation, Figure \ref{fig:motivation} plots the VP stochastic trajectories (the green curves) connected with our FP-Diffusion forward trajectories (the white curves) under random initialization. We also provide an informal argument on how the anisotropic FP-Diffusions mix with the low-dimensional data distribution in Appendix A.

In fact, to \textbf{unify} the critical damped Langevin diffusion model \citep{dockhorn2022score}, FP-Diffusion is straightforward to generalize to the case when the inverse Riemannian matrix $R^{-1}(x)$ degenerates (contains zero eigenvalues). However, $X_t$ may not converge to this Gaussian stationary distribution from a deterministic starting point. Intuitively, the diffusion part $\sqrt{R^{-1}(x)}dW_t$ is the source of randomness (noise). Suppose $R^{-1}(x)$ degenerates along the $i$-th direction (i.e., corresponding to zero eigenvalue), then no randomness is imposed on this direction, then the $i$-th component $X_t^i$ will be frozen at $X_0^i$. To remedy the issue, we impose additional restrictions, which lead us to the following corollary: 
\begin{corollary} \label{coroll}
Under additional conditions: (1) the symplectic form $\omega \in \mathbb{R}^{2d \times 2d}$ has the block form:
$\omega =
\left(\begin{array}{@{}c|c@{}}
  \begin{matrix}
  0
  \end{matrix}
  & \ A\  \\
\hline
  -A &
  \ 0\ 
\end{array}\right)
$, with a positive-definite matrix $A \in \mathbb{R}^{d \times d}$; (2) the inverse (semi-) Riemannian matrix $R^{-1}(x)$ has the block form:
$R^{-1}(x) =
\left(\begin{array}{@{}c|c@{}}
  \begin{matrix}
  \ 0
  \end{matrix}
  & 0  \\
\hline
  \ 0 &
  B
\end{array}\right) \label{damped}
$, with a constant positive-definite symmetric matrix $B \in \mathbb{R}^{d \times d}$, it induces that the forward diffusion $X_s$ converges to the standard Gaussian distribution:
$$p_s(X_s) \xrightarrow{s \rightarrow \infty} \mathcal{N}(0,I).$$
\end{corollary}
We will demonstrate how the corollary derives the damped diffusion model in Appendix A.
\subsection{Parameterizing diffusion models from the optimization perspective}
\label{main:param_forward}
In this section, we illustrate the benefits of parameterized diffusion models from the view of the training objective. Recall that the ground-truth reverse-time (generative) SDE of the forward process $X_t$ is denoted by $Y_t$, 
and we parameterize $Y_t$ by $Y^{\theta}_t$ :
\begin{equation} \label{gener}
dY^{\theta}_t = [f(Y_t,t) - g^2(Y_t,t)\nabla \mathbf{s}_{\theta}(Y_t,t)]dt + g(t)dW_t,   \end{equation}
where $\mathbf{s}_{\theta}$ is the score neural network parameterized by $\theta$. Then the (explicit) score-matching loss function for optimization is
\begin{equation} \label{loss1}
L_{\text{ESM}} := \int_0^T \mathbb{E}_{X_s}[\frac{1}{2}\norm{\mathbf{s}_{\theta}(X_s , s) - \nabla \log p_s(X_s)}^2_{\Lambda (s)}]ds,    
\end{equation}
where $\Lambda (s)$ is a positive definite matrix as a weighting function for the loss. Since $Y_t$ and $X_t$ shares the same marginal distributions, suppose our parameterized generative process $Y^{\theta}_t$ matches $Y_t$ perfectly, namely 
\begin{equation} \label{gap}
\mathbf{s}_{\theta}(x,t) \equiv \nabla \log p_t(x)    
\end{equation}
for all $t \in [0,T]$, then
the marginal distribution of $Y^{\theta}_t$ at $t = 0$ is exactly the data distribution.

The major obstacle of optimizing Eq. \ref{loss1} directly is that we don't have access to the ground truth score function $\nabla \log p_s(x,s)$. Fortunately, $L_{\text{ESM}}$ can be transformed to a loss based on the accessible
conditional score function $\nabla \log p_{X_s |X_0}(X_s)$  plus a constant \citep{song2020sliced,song2020score} (\textbf{for a fixed forward process $X_s$}). More precisely, given two time slices $0 < s < t < T$,
\begin{align} \label{conditional}
\mathbb{E}_{X_t}\norm{\mathbf{s}_{\theta}(X_t , t) - \nabla \log p_t(X_t)}^2 \equiv&\     \mathbb{E}_{X_s, X_t} \norm{\mathbf{s}_{\theta}(X_t , t) - \nabla \log p_t(X_t|X_s)}^2  \\
& + \underbrace{\mathbb{E}_{X_t}\norm{\nabla \log p_t(X_t)}^2 - \mathbb{E}_{X_s, X_t}\norm{\nabla \log p_t(X_t|X_s)}^2}_{\text{gap terms}} .
\end{align}
Since the gap terms between the absolute and conditional score function loss do not depend on the backward generative process, one \textbf{theoretical advantage} of FP-Diffusion is that the gap terms are also parameterized. 
This formula is adapted from \citep{song2020sliced,huang2021variational} by modifying the initial time, and full derivations are given in Appendix A for completeness.

On the other hand, compared with log-likelihood generative models like normalizing flows and VAE, the connection between score matching and the log-likelihood of data distribution $\log p_0(x)$ is not straightforward due to the additional forward noising process $X_t$. 
Hence, we turn to  the variational view established in \citep{huang2021variational}, where the ELBO (evidence lower bound) of data's log-likelihood $\log p_0(x)$ is related with the score matching scheme. More precisely, 
$$\log p_0(x) \ge \mathcal{E}^{\infty}(x),$$
and the ELBO $\mathcal{E}^{\infty}(x)$ of the infinite-dimensional path space is defined by 
$$\mathcal{E}^{\infty}(x) : = \mathbb{E}_{X_T}[\log p_T(X_T)|X_0 = x] - \int_0^T \mathbb{E}_{X_s}[\frac{1}{2}\norm{\mathbf{s}_{\theta}}_{g^2}^2 + \nabla \cdot (g^2 \mathbf{s}_{\theta} - f) | X_0 = x]ds. $$ 
The above implies that learning a diffusion (score) model is equivalent to maximizing the ELBO in the variational path space defined by the generative process $Y_t^{\theta}$. Thus, treating $f(x,t)$ and $g(x,t)$ as learnable functions results in enlarging the variational path space from pre-fixed $f$ and $g$ to flexible variational function classes, and the domain of ELBO (parameterized by $f$ and $g$) is extended correspondingly. 

By Eq. \ref{sum}, in FP-Diffusion model, we set \begin{equation} \label{fg}
f(x,t): =  \frac{\beta ' (t)}{2}[-\sum_j R^{-1}_{ij}(x) x_j -  2\sum_j \omega_{ij} x_j
+ \sum_j \frac{\partial}{\partial x_j} R^{-1}_{ij}(x)],\ \ g(x,t) := \sqrt{\beta ' (t) R^{-1}(x)}.\end{equation}
Since our variational function class of the forward process defined in Eq. \ref{sum} is theoretically guaranteed to approach Gaussian when $T$ is large, the first term of $\mathcal{E}^{\infty}(x)$ is close to a small constant under Eq. \ref{fg}. Therefore, we only need to investigate the second term (equivalent to the implicit score matching~\citep{hyvarinen2005estimation}), which depends on both the parameterized $f$, $g$ and the score function. 
Finally, learning $f$ and $g$ opens the opportunity of adding additional \textbf{regularization penalties} to filter out irregular paths in the extended variational path space. Similar techniques had been applied in continuous normalizing flows \citep{finlay2020train}.
Preliminary exploration on 
applying regularization to FP-Diffusion models is clarified in the experimental section.  

\subsection{A simplified formula of FP-Diffusion}
Although we can always numerically simulate the SDE to sample at a given time $t$,  the empirical success of the Monte-Carlo training of (14) in \citep{ho2020denoising} (see also (7) of \citep{song2020score}) indicates the importance of deriving explicit solutions for the SDE. In this section, we derive the solution formula for a simplified version of $X_t$ defined in Eq. \ref{sum} and implement it on the image generation task.

To obtain the closed-form of the transition probabilistic density function of the forward process $X_t$, we assume that $R^{-1}(x)$ of Eq. \ref{sum} is a constant symmetric positive-definite matrix independent of the spatial variable $x$. Then in the linear SDE region \citep{S2019Applied}, we have the following characterization of the marginal distributions (see Appendix A for a full derivation):
\begin{theorem} \label{simple version}
Suppose the forward diffusion process $X_t$ starting at $X_0$ satisfies the following linear stochastic differential equation:
\begin{equation} \label{simple}
d X_t = \frac{1}{2}\beta'(t)[-R^{-1}X_t - 2 \omega X_t]dt + \sqrt{\beta'(t) R^{-1}} dW_t,    
\end{equation}
for positive-definite symmetric $R$ and anti-symmetric $\omega$. 
Then the marginal distribution of $X_t$ at arbitrary time $t>0$ follows the Gaussian distribution:
$$X_t \sim \mathcal{N}(e^{(-\frac{1}{2}R^{-1} - \omega)\beta(t)}X_0, \mathbf{I} - e^{- \beta(t) R^{-1}}).$$
\end{theorem}
To treat the drift and diffusion terms in Eq. \ref{sum} separately, we name Eq. \ref{metric} as the \textbf{FP-Drift}, Eq. \ref{sym} as the \textbf{FP-Noise}. In practice, we parameterize the anti-symmetric matrix $\omega$ in the FP-Drift model and the anti-symmetric matrix $R^{-1}$ in the FP-Noise model. Both the anti-symmetric and symmetric matrices are realized through the exponential map. We leave the implementation details in Appendix A.






  












  







\section{Experiment}
We first use a synthetic 3D dataset to illustrate the significance of parameterizing the forward process according to the data distribution, then validate the effectiveness of our FP-Diffusion model on the image generation task.
\subsection{Flexible SDEs learned from Synthetic 3D examples}
\label{toy}
\begin{figure}[t]
\centering
    \begin{subfigure}[t]{0.32\textwidth}
    \centering
    \includegraphics[width=\textwidth]{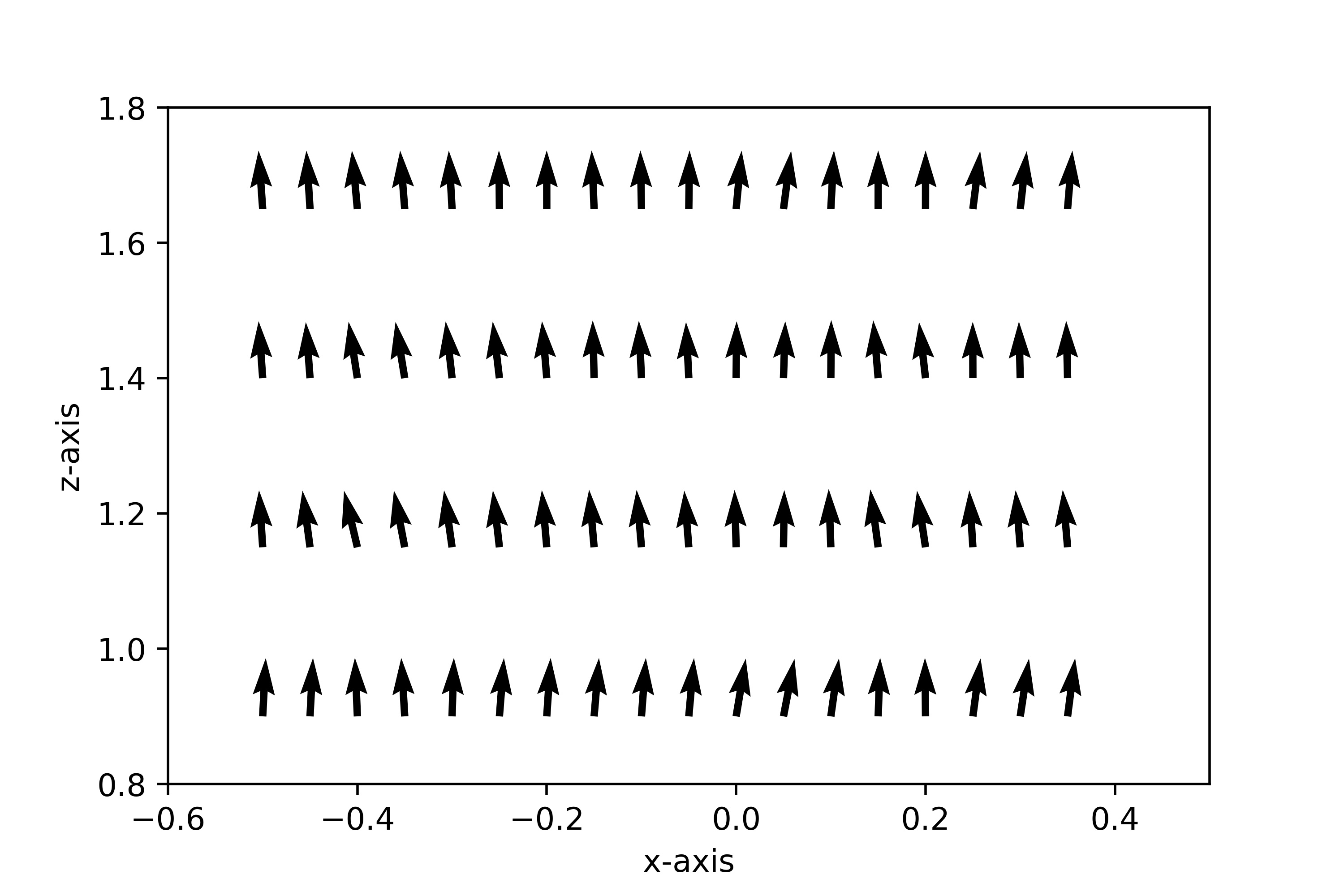}
    \caption{VP}
    \end{subfigure}
    \begin{subfigure}[t]{0.32\textwidth}
    \centering
    \includegraphics[width=\textwidth]{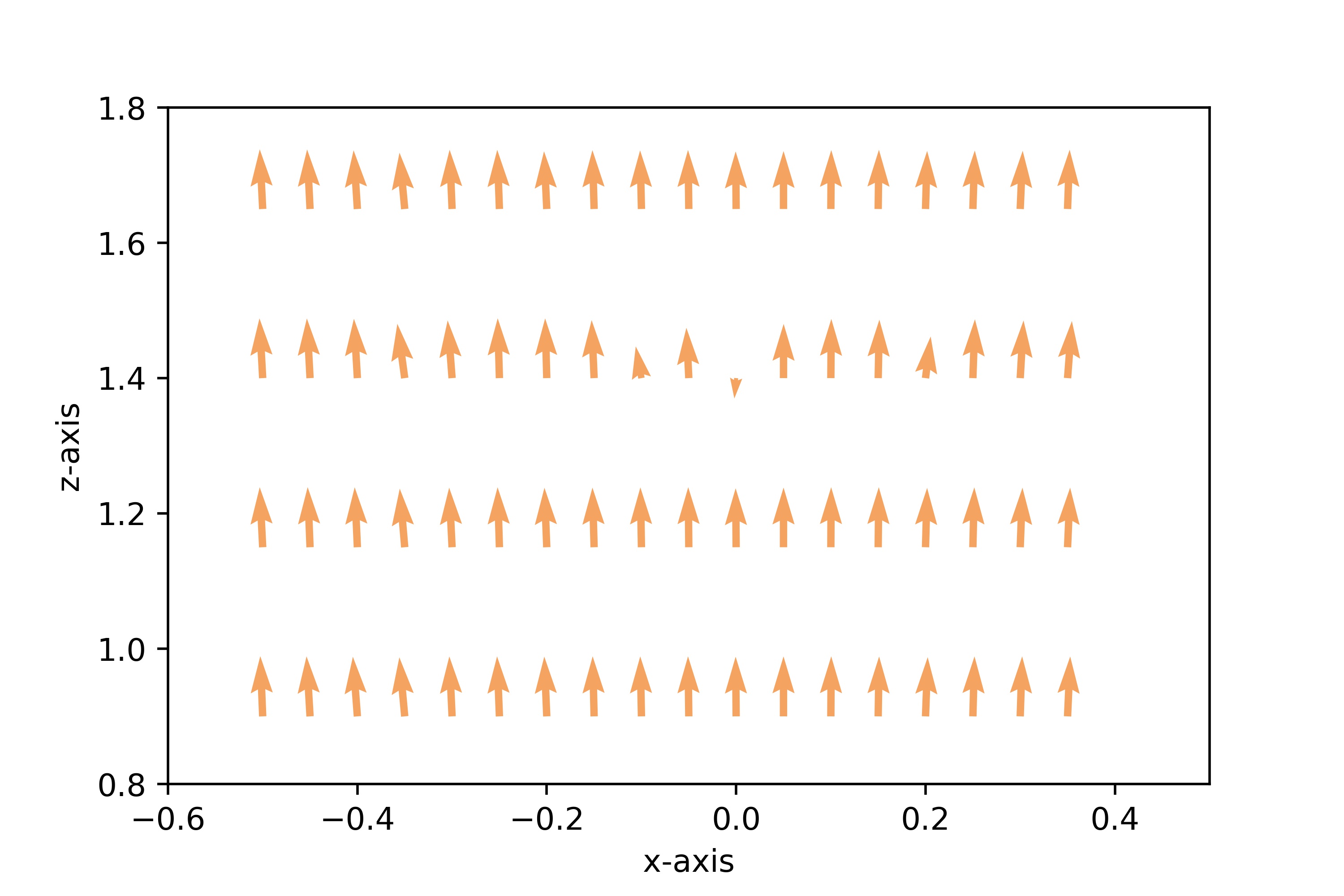}
    \caption{Non-reg. FP}
    \end{subfigure}
    \begin{subfigure}[t]{0.32\textwidth}
    \centering
    \includegraphics[width=\textwidth]{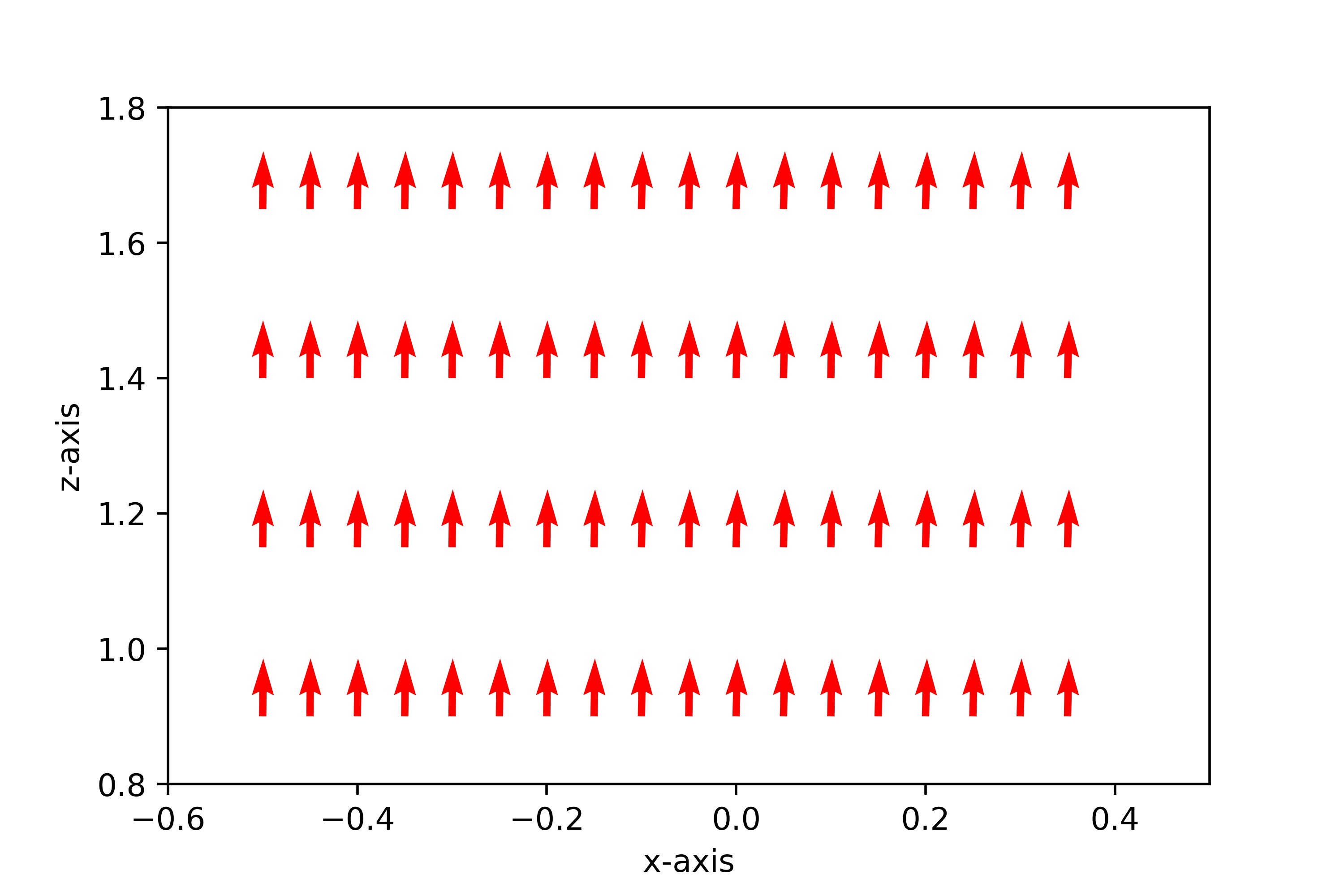}
    \caption{Reg. FP}
    \end{subfigure}
\caption{Vector fields projected into 2D-section of three SDEs}
\label{fig:toy1}
\vspace{-0.4cm}
\end{figure}

By the low-dimensional manifold hypothesis \citep{fefferman2016testing}, the real data distribution concentrates on a low-dimensional submanifold. However, during the generation phase, the dimension of the ambient space we sample from is much higher. 
In this case, FP-Diffusion plays a nontrivial role. To be more precise, note that only the diffusion part of Eq. \ref{sum} can blur the data submanifold to fill in the high-dimensional ambient space, which causes a distinction between the directions tangent to the data manifold and the normal directions for (anisotropic) diffusing. Since it is impossible to detect the data manifold of a complex dataset, we design a simplified scenario to demonstrate how the parameterized diffusion process helps generation.
\begin{figure}[htp]
    \centering
    \begin{minipage}{0.32\textwidth}
        \makeatletter\def\@captype{table}\makeatother\caption{NLLs on MNIST}
        \label{table:mnist}
        \centering
            \begin{tabular}{ccccc}
            \toprule
            Model        & NLL $\downarrow$ \\
            \midrule
            RealNVP \citep{dinh2016density} & 1.06 \\ 
            Glow \citep{kingma2018glow} & 1.05 \\ 
            FFJORD \citep{grathwohl2018ffjord} & 0.99 \\ 
            ResFlow \citep{chen2019residual} & 0.97 \\ 
            DiffFlow \citep{zhang2021diffusion} & 0.93 \\
            \midrule
            FP-Drift (Mix) & 1.01 \\ 
            \bottomrule
            \end{tabular}
    \end{minipage}
    \hspace{0.1cm}
    \begin{minipage}{0.32\textwidth}
        \centering
        \includegraphics[width=\textwidth]{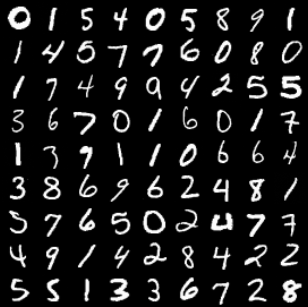}
        \caption{MNIST samples}
        \label{figure:mnist}
    \end{minipage}
    \hspace{0.1cm}
    \begin{minipage}{0.32\textwidth}
        \centering
        \includegraphics[width=\textwidth]{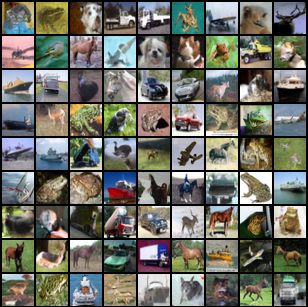}
        \caption{CIFAR10 samples}
        \label{figure:cifar}
    \end{minipage}
    \vspace{-0.3cm}
\end{figure}


\begin{table}[]
    \centering
    \caption{Results on CIFAR10. * denotes the results we reproduce locally.}
    \label{table:cifar10}
    \begin{tabular}{ccc}
            \toprule
            Model           & FID $\downarrow$ & NLL$\downarrow$ \\
            \midrule
            DDPM++ cont. (deep, VP)~\citep{song2020score} & 2.95* & 3.13*\\
            NCSN++ cont. (deep, VE)~\citep{song2020score} & 2.72*  &  - \\
            DDPM~\citep{zhang2021diffusion} & 3.17 & $\leq\text{3.75}$ \\
            Improved-DDPM~\citep{nichol2021improved} & 2.90 & 3.37\\
            LSGM~\citep{vahdat2021score} & 2.10 & $\leq\text{3.43}$\\
            LSGM-100M~\citep{dockhorn2022score} & 4.60 & $\leq\text{2.96}$\\
            CLD-SGM~\citep{dockhorn2022score} & 2.25 & $\leq\text{3.31}$\\
            DiffFlow~\citep{zhang2021diffusion} & 14.14 & 3.04 \\
            \midrule
            FP-Drift (Joint) & 4.17 & 3.30 \\
            FP-Noise (Joint) & 3.30 & 3.25\\
            FP-Drift (Mix) & 2.99 & 3.28 \\
            FP-Noise (Mix) & 2.87 & 3.20\\
            \bottomrule
    \end{tabular}
    \vspace{-0.3cm}
\end{table}
            
Assume the data lies in $\mathbf{R}^3$, and its probabilistic distribution intrinsically follows the 2-dimensional Gaussian concentrated at a hyper-plane. The most efficient way to generate the 2-dimensional Gaussian is to project the random points sampled from 3-dimensional Gaussian to this plane. 
To quantify this fact, we consider the optimal transport problem from the 3-dimensional Gaussian distribution to the 2-dimensional Gaussian.
Define the cost function as $c(x,y) := \norm{x-y}^2$, then the the Wasserstein distance between $\mathcal{N}(0,\mathbf{I})$ and $\mathcal{N}(\mu,\Sigma)$ \citep{mallasto2017learning} equals $\mathcal{W}_2(\mathcal{N}(0,\mathbf{I},\mathcal{N}(\mu,\Sigma)) = \norm{\mu}^2 + \text{Tr} (\mathbf{I} + \Sigma -2\Sigma^{1/2}).$ It implies that corresponding optimal transport map is exactly the vertical projection map, denoted by $\nabla \phi$. 

For FP-Diffusion, the probabilistic flow of the generating process depends on the learnable forward process, 
which allows us to add regularization penalty terms to the score-matching loss to choose preferable paths. From (13) of \citep{song2020score}, the vector field of the probability flow ODE is
\begin{equation} \label{pc}
v_{pc}(x,t) : = f(x,t) -  \frac{1}{2}\nabla \cdot[g^2(x,t)] - \frac{1}{2}g^2(x,t) \nabla \log p_t (x), \\ \ x \in \mathbf{R}^3.    
\end{equation}

Now we check whether the direction of the learned $v_{pc}$ is aligned with the ground-truth projection vector field. For the projection map $\nabla \phi$ from the 3-dimensional Gaussian to 2-dimensional Gaussian supported at the plane :$z = 2$, the corresponding vector field at a spatial point $x=(x,y,z) \in \mathbb{R}^3$ is:
\begin{equation} \label{ground}
v_{\text{proj}}(x,t) : = (0,0,-1) \ \text{if}\ \ z >2, \ \ \text{and}\ \ v_{\text{proj}}(x,t) : = (0,0,1) \ \text{if}\ \ z <2.\end{equation}
We perform the experiment under three circumstances: 1. a fixed VP forward process (Eq. \ref{ou}) diffusion model; 2. our parameterized forward process with no regularization; 3. our parameterized forward process with regularization terms. The  regularization penalties are imposed on the vector field (Eq. \ref{pc}), which are adapted from section 4 of \citep{finlay2020train}: $L_{reg}(f,g) = \lambda_1 \int \norm{v_{pc}(s)}^2 ds + \lambda_2 \mathbb{E}_{\epsilon \sim \mathbf{N}(0,1)} \norm{\epsilon^Tv_{pc}(s)}^2 ds.$
This term is only designed for regularizing the parameterized $f$ and $g$ of the \textbf{forward process}. 

The 2D visualization results of our comparative experiments are summarized in Fig. \ref{fig:toy1}. Notice that the `ground-truth' vector field (Eq. \ref{ground}) is strictly vertical, so we plot the $x-z$ projection of the trained three vector fields at a given time for the three scenarios. From Fig. \ref{fig:toy1}, although our flexible diffusion method (b) is visibly more vertical than the forward-fixed VP (ddpm) model (a), the flexible model with regularization (c) is more close to vertical lines. We also sample the integration trajectories of the trained vector fields for comparison in Appendix B (see Figure 5). 

\subsection{Image Generation}
In this section, we demonstrate the generative capacity of the score models driven by our FP-Diffusion on two common image datasets: MNIST \citep{lecun1998mnist} and CIFAR10 \citep{krizhevsky2009learning}.

\textbf{Training strategy. }The flexible FP-Diffusion framework is designed to simultaneously learn a suitable forward diffusion process dependent on the data distribution as well as the corresponding reverse-time process. However, 
for some complex scenarios like image generation, it is challenging to balance the optimization of the FP-Diffusion model and the score model. To compromise these two parts, we propose a two-stage training strategy. Particularly, in the first stage, we jointly optimize the parameters of both the FP-Diffusion model and the score neural network; in the second stage, we fix all parameters in the FP-Diffusion and only tune the score neural network as the prevailing score-based MCMC approaches \citep{ho2020denoising, song2019generative, song2020score}.

\textbf{Implementation Details. } For the forward diffusion process, we choose a linearly increasing time scheduler $\beta(t)$ (same as the VP-SDE setting in \citep{song2020score}), where $t \in [0, T]$ is a continuous time variable. To estimate the gradient vector field in the reverse-time process, we train a time-dependent score network $s_\theta(x(t), t)$ as described in Eq. \ref{conditional}.
We adopt the same U-net style architecture used in \citep{ho2020denoising} and \citep{song2020score} as our score network. We train both the FP-Drift model and the FP-Noise model in two training paradigms: 1) \textbf{Joint Training}: the parameterized FP-Diffusion model and the score network are jointly optimized for 1.2M iterations; 2) \textbf{Mix Training}: following the proposed two-stage training strategy, we separately train the model for 600k iterations in both stages. The batch size is set to 96 on all datasets. We apply the Euler-Maruyama method in our reverse-time SDEs for sampling images, where the discretization steps are set to 1000 as in \citep{song2020score}. All the experiments are conducted on 4 Nvidia Tesla V100 16G GPUs. We provide further implementation details in Appendix B.

\textbf{Results. } We show the sampled images generated by our FP-Noise (Mix) model in Fig. \ref{figure:mnist} and Fig. \ref{figure:cifar}.
According to Eq. \ref{pc}, we calculate the negative log-likelihood (NLL) in bits per dimension for our models by the instantaneous change of variables formula \citep{grathwohl2018ffjord}. Then we list the NLL metrics of our models in Tab. \ref{table:mnist} and Tab. \ref{table:cifar10}. On MNIST, our FP-Drift model achieves comparable performance in terms of NLL, compared to five flow-based models (including DiffFlow \citep{zhang2021diffusion}). On CIFAR10, both the FP-Drift (Mix) and the FP-Noise (Mix) models achieve a competitive performance compared to the state-of-the-art (SOTA) diffusion models. These results illustrate the strong capacity of FP-Diffusion in density estimation tasks.

To quantitatively evaluate the quality of sampled images, we also report the Fenchel Inception Distance (FID) \citep{heusel2017gans} on CIFAR10. As shown in Tab. \ref{table:cifar10}, the two variants of our FP-Diffusion model, FP-Drift (Mix) and FP-Noise (Mix), outperform DDPM \citep{ho2020denoising} and Improved-DDPM \citep{nichol2021improved} in FID and have a comparable performance with DDPM++ cont. (deep, VP) and NCSN++ cont. (deep, VE) \citep{song2020score}. We notice that only LSGM and CLD-SGM have obviously better FID values than other models (including us). However, LSGM \citep{vahdat2021score} adopts a more complicated framework and a large model with $\approx 475M$ parameters to achieve its high performance. With a comparable parameter size ($\approx 100M$), our models could achieve a significantly better FID score than LSGM (``LSGM-100M''). CLD-SGM builds its diffusion model upon a larger phase space with a special training objective (given the data point $x \in \mathbb{R}^n$, its phase space corresponding point $(x,v)$ belongs to $\mathbb{R}^{2n}$), which leads to a more expressive optimization space but brings extra computational cost as well. We leave testing our FP-Diffusion model on phase space (defined in Corollary \ref{coroll}) in the future. It should also be noted that we use a smaller batch size ($96$) compared to other baseline diffusion models ($128$) to train our models due to limited computational resources, which may influence our performance to some extent. We also report the performance of our two model variants in two training paradigms in Tab. \ref{table:cifar10}. The model variants with the joint training paradigm consistently achieve a better performance, demonstrating the necessity of the two-stage training strategy. A possible reason of this phenomenon is that it is difficult for score models to match the reverse process of a dynamical forward process, so we need to re-train the score model with extra training steps after learning a suitable forward process.
\section{Conclusion}
We propose FP-Diffusion model, a novel method that parameterizes the spatial components of the diffusion (score) model with theoretical guarantees. This approach combines insights from Riemannian geometry and Hamiltonian Monte-Carlo methods to obtain a  flexible forward diffusion framework that plays a nontrivial role from the variational perspective. Empirical results on specially-designed datasets and standard benchmarks confirm the effectiveness of our method. However, efficiently optimizing FP-Diffusion remains a critical challenge, which opens the door for promising future research.

\bibliography{neurips_2022}
\bibliographystyle{plain}
\appendix

\section{Theory}

\subsection{Discussion of Section 3.1} \label{Section 3.1}
\subsubsection{Remark on the theoretical properties of Hamiltonian dynamics}
Suppose $X_t$ follows the Hamiltonian dynamics (\ref{sym}), then
$$dH(X_t) = \nabla H(X_t) \omega \nabla H(X_t) dt \equiv 0,$$
by the anti-symmetry of $\omega$. Therefore, the Hamiltonian dynamics without random perturbations is a deterministic motion that can explore within a constant Hamiltonian (energy) surface. It means that, only by adding a diffusion term, the Hamiltonian dynamical system is able to traverse different energy levels. 

\subsubsection{Verification of Theorem \ref{general}}
We will verify the theorem under the more general case, when $H(x) = \frac{m}{2}x^2$. The corresponding stationary distribution is the scaled Gaussian $\mathcal{N}(0,m\mathbf{I})$, where $m >0$ is the scale constant. In this case, Eq. \ref{sum} is modified to:

\begin{align} \nonumber
dX_t = \frac{m}{2}[-\sum_j R^{-1}_{ij}(X_t)\cdot(X_t)_j - & 2\sum_j \omega_{ij}\cdot(X_t)_j
+ \sum_j \frac{\partial}{\partial x_j} R^{-1}_{ij}(X_t)] dt\\ + & \sqrt{R^{-1}(X_t)}dW_t.  \label{sum2} \end{align}    
Note that only the drift term is scaled by $m$.
\begin{proof}
Since the covariance matrix of the diffusion part is positive-definite, the forward process Eq. \ref{sum2} satisfies the Feller property and the existence and uniqueness of the stationary distribution are guaranteed (see \citep{2014Analysis}). 
By the Fokker-Plank-Kolmogorov equation, the stationary distribution $p_s(x)$ of Eq. \ref{sum2} should satisfy
\begin{equation} \label{hold}
0 = -\sum_i \frac{\partial}{\partial x_i}[f_i(x,t)p_s(x)] + \frac{1}{2}\frac{\partial^2}{\partial x_i \partial x_j}[(gg^T)_{ij}p_s(x)],    
\end{equation}
where we set $f(x,t):= \frac{m}{2}[-\sum_j R^{-1}_{ij}(x)\cdot x_j - 2\sum_j \omega_{ij}\cdot x_j 
+ \sum_j \frac{\partial}{\partial x_j} R^{-1}_{ij}(x)]$ and $g(x,t) := \sqrt{R^{-1}(x)}$. To check whether $e^{- \frac{m}{2}x^2}$ satisfies condition (\ref{hold}), notice that by the anti-symmetry of $\omega_{ij}$, we automatically have
\begin{align*}
\sum_i\sum_j \frac{\partial}{\partial x_i} ( \omega_{ij} x_j e^{-\frac{m}{2}x^2})  = - \sum_i\sum_j \omega_{ij} x_i x_j e^{-\frac{m}{2}x^2}=0.
\end{align*}
On the other hand,
\begin{align*} 
\sum_i\sum_j \frac{\partial^2}{\partial x_i \partial x_j}[R^{-1}_{ij}(x)e^{-\frac{m}{2}x^2}] &= -m \text{Tr}(R^{-1}_{ij}(x))e^{-\frac{m}{2}x^2}
+ m^2 \sum_i\sum_j R^{-1}_{ij}(x) x_i x_j e^{-\frac{m}{2}x^2}\\
& - \sum_i\sum_j \frac{\partial}{\partial x_j}(R^{-1}_{ij}(x))(\frac{\partial}{\partial x_i}e^{-\frac{m}{2}x^2})\\
&+ \sum_i\sum_j \frac{\partial^2}{\partial x_i \partial x_j}(R^{-1}_{ij}(x))e^{-\frac{m}{2}x^2}\\
& -m \sum_i\sum_j \frac{\partial}{\partial x_j}(R^{-1}_{ij}(x))x_ie^{-\frac{m}{2}x^2}\\
=& -m\text{Tr}(R^{-1}_{ij}(x))e^{-\frac{m}{2}x^2}
+ m^2\sum_i\sum_j R^{-1}_{ij}(x) x_i x_j e^{-\frac{m}{2}x^2}\\
&+ \sum_i \frac{\partial}{\partial x_i}[\sum_j \frac{\partial}{\partial x_j}R_{ij}^{-1}(x) e^{-\frac{m}{2}x^2}] \\
& - m \sum_i\sum_j \frac{\partial}{\partial x_j}R^{-1}_{ij}(x)x_ie^{-\frac{m}{2}x^2}.
\end{align*}
Therefore, the last thing to check is that
\begin{align*}
\sum_i & \frac{\partial}{\partial x_i} [\sum_j R^{-1}_{ij}(x) x_j e^{-\frac{m}{2}x^2}] =  \text{Tr}(R^{-1}_{ij}(x))e^{-\frac{m}{2}x^2}-\\
& \sum_i\sum_j[ m R^{-1}_{ij}(x) x_i x_j  +  \frac{\partial}{\partial x_j}R^{-1}_{ij}(x)x_i]e^{-\frac{m}{2}x^2},
\end{align*}
which is obviously true, since the diffusion matrix $R^{-1}_{ij}$ is symmetric. Combining the above, we have proved that Eq. \ref{hold} holds if $p_s(x) \propto e^{-\frac{m}{2}x^2}$.
\end{proof}
\subsubsection{Completeness of FP-Diffusion parameterization}
From the last section's derivation, we can deduce the following corollary:
\begin{corollary} \label{complete}
Consider the following SDE: 
$$dX_t = A(X_t)dt - \frac{1}{2}R^{-1}(X_t) \cdot X_t dt  + (\nabla \cdot R^{-1}(X_t)) \cdot X_t dt + \sqrt{R^{-1}(X_t)}dW_t,$$
and let the spatial function $A(x)$ be a linear function. Suppose we know its
stationary distribution is standard Gaussian, then
$$A(x) =  - \sum_j \omega_{ij}\cdot x_j 
,$$
for some anti-symmetric matrix $\omega$.
\end{corollary}
\begin{proof}
In fact, every linear operator $A$ can be decomposed into a symmetric part plus an anti-symmetric part:
$$A = \underbrace{\frac{A + A^T}{2}}_{\text{symmetric}} + \underbrace{\frac{A - A^T}{2}}_{\text{anti-symmetric}}.$$
Let $\omega = \frac{A - A^T}{2}$. Then we only need to prove that $A + A^T$ equals zero, if $X_t$ converges to Gaussian. 

From the proof of theorem 3.3, we extract the fact that if $p_s(x) \propto e^{-\frac{1}{2}x^2}$,
$$\sum_i \frac{\partial}{\partial x_i}[(A + A^T)_{ij}\cdot x_j e^{- \frac{1}{2}x^2}] = 0,$$
then
$$\sum_{i,j}[(A + A^T)_{ij}\cdot  \frac{\partial}{\partial x_i}\frac{\partial}{\partial x_j}(e^{- \frac{1}{2}x^2})] = 0,$$
for all $x =(x_1,\dots,x_n)$.
Note that 
$$\frac{\partial}{\partial x_i}\frac{\partial}{\partial x_j}(e^{- \frac{1}{2}x^2}) = (x_ix_j - \delta_{ij} )e^{- \frac{1}{2}x^2}.$$ 
Since $A + A^T$ is symmetric (\textbf{doesn't hold for arbitrary linear operator}), it implies that $A+A^T \equiv 0$.
\end{proof}
\subsubsection{Anisotropic diffusion on low dimensional data manifold}
In this section, we give an informal discussion on how an anisotropic diffusion starting at a low-dimensional data manifold mixes with its own stationary distribution (supported in the high dimension ambient space). 

Assume the marginal distribution of the diffusion process $X_t$ concentrates on a low dimensional manifold $M \hookrightarrow \mathbb{R}^n$ at a given time. Moreover, suppose $X_t$ already achieves the Gaussian stationary distribution on $M$ (defined with respect to the Laplacian operator of $M$).  Now we want to informally investigate the most efficient way for $X_t$ to diffuse out of the low dimensional sub-manifold to the ambient space. By localizing in the Riemannian normal coordinates and by arranging the coordinates indexes, we can further assume that $M$ is isometric to the hyperplane of $\mathbb{R}^n$ defined by 
$$ M = \{x \in \mathbb{R}^n | x = (x_1,\dots,x_p,0,\dots,0)\}.$$
Then the coordinate components of each point $x \in \mathbb{R}^n$ can be decomposed into the tangential directions and the normal directions with respect to $M$:
$$x \in (\underbrace{x_1,\dots,x_p}_{\text{tangent to M}},\underbrace{x_{p+1},\dots,x_n}_{\text{normal to M}}).$$
Under the above conditions, we are ready to compare the convergence rate (to the high-dimensional stationary Gaussian distribution of $\mathbb{R}^n$) of different forward diffusions defined in (\ref{metric}). For a fair comparison, we set the norm of the noise matrix to be one: $\norm{R^{-1}}_2 \equiv \sqrt{n}$. Otherwise, the convergence can always be accelerated by increasing the noising scale ($\norm{R^{-1}}_2 \rightarrow \infty$).

Under our normal coordinates, the forward diffusion can be decomposed into two parts :$X(t) = X_{tan}(t) + X_{nor}(t)$. For simplicity, suppose $R^{-1}$ is a diagonal matrix, then the tangential part and the normal part of $X(t)$ is completely decoupled. In other word,
$$X_{tan}^i(t) = \frac{1}{2}[- R^{-1}_{ii}\cdot(X_t)^i
] dt + \sqrt{R^{-1}_{ii}}dW_t^i, \ \ 1 \leq i \leq p$$
is a diffusion process on $M$. Therefore, $(X_{tan}(t), X(t))$ is indeed a Markov coupling. Suppose $X_{tan}(t)$ at $t=0$ already converges to its stationary distribution (low dimensional Gaussian), then by Ito's formula,
\begin{align*}
d (X_{tan}(t) - X(t))^2 &= d X_{nor}^2(t)\\
& = 2X_{nor}(t)(\frac{1}{2}[- R^{-1}_{nor}\cdot X_{nor}(t)
] dt + \sqrt{R^{-1}_{nor}} dW_t) + \text{Tr}(R^{-1}_{nor})dt.
\end{align*}
Taking the expectation of both sides, it implies that
$$\frac{d \mathbb{E} X_{nor}^2(t)}{dt} = - \mathbb{E} R^{-1}_{nor}\cdot X_{nor}^2(t) +  \text{Tr}(R^{-1}_{nor}).$$
Let $r_{min}$ denote the minimal eigenvalue of the normal part of $R^{-1}$, then
$$\frac{d \mathbb{E} X_{nor}^2(t)}{dt} \leq - r_{min} \mathbb{E} X_{nor}^2(t) +  \text{Tr}(R^{-1}_{nor}).$$
Applying Grönwall's inequality and note that $X_{nor}(0) = 0$, we have
$$\mathbb{E} X_{nor}^2(t) \leq e^{-r_{min} \cdot t} \cdot \text{Tr}(R^{-1}_{nor})t.$$
The above gives an upper bound on the convergence speed of the coupling $(X_{tan}(t), X(t))$ with respect to the $W_2$ distance (see \cite{2014Analysis}). Since the stationary distribution of $X(t)$ is exactly the high dimensional Gaussian distribution (the diffusion model's prior distribution), we hope the convergence rate to be as fast as possible (given a fixed noising scale). For the VP-Diffusion,
$$R^{-1}_{nor} \equiv \text{diag} \{1, \dots, 1\}.$$
 However, in FP-Diffusion model, the diagonal elements of $R^{-1}_{nor}$ are allowed to be inhomogeneous and greater than one (under the condition that $\text{Tr}(R^{-1}_{nor}) < n$). This will lead to a smaller $r_{min}$, which will speed up the convergence rate by our analysis.
\subsubsection{Verification of Corollary \ref{coroll}}
The intuition of Corollary \ref{coroll} can be stated as follows: To guarantee the geometric ergodicity property of FP-Diffusion on the \textbf{phase space}, we need enough noise such that the diffusion process can transverse the whole space. Suppose $R^{-1}(x)$ degenerates along the $i$-th direction (corresponding to a zero eigenvalue), then no randomness (noise) is imposed on this direction.

To remedy the issue, we require the symplectic form $\omega$ to be non-zero along the i-th direction, which makes it possible to mix the noise originated along other directions (where $R^{-1}(x)$ is strictly positive-definite) with the i-th direction. Now we give the formal proof:
\begin{proof}
We only prove for the simplified case when $A$ and $B$ are both diagonal matrices with two sets of positive eigenvalues $\{a_i\}_{i=1}^d$, $\{b_i\}_{i=1}^d$. The general situation can be handled by trivial linear transformation. By 
proposition 8.1 of \citep{bellet2006ergodic}, the proof boils down to prove that the Hörmander's condition \citep{hormander1967hypoelliptic} holds for the forward process $X_t$. When $R^{-1}(x)$ is a constant matrix, the infinitesimal generator $L$ of (\ref{sum}) is:
$$L = \sum_i \sum_j \frac{1}{2}[-R^{-1}_{ij} - 2 \omega_{ij}]x_j \frac{\partial}{\partial x_j} + \frac{1}{2}\sum_{ij}R^{-1}_{ij} \frac{\partial^2}{\partial x_i \partial x_j}.
$$
For notation simplicity, denote $x:= (u,v) \in \mathbb{R}^{2d}$, where $u,v \in \mathbb{R}^{d}$. 
To put the second-order differential operator $L$ in Hörmander's form, set 
$$Y_j (u,v) = - \frac{1}{2}\sqrt{b_j} \frac{\partial}{\partial v_j},\ \ 1 \leq j \leq n,$$
and
$$Y_0 (u.v) = \sum_i (-a_i v_i \frac{\partial}{\partial u_i} + a_i u_i \frac{\partial}{\partial v_i}).$$
Then it suffices to show that the vector fields $\{[Y_0,Y_j],Y_j\}_{1 \leq j \leq d}$ span the whole $\mathbb{R}^{2d}$. By direct calculation,
$$[Y_0,Y_j] = \frac{1}{2}a_j\sqrt{b_j} \frac{\partial}{\partial u_j},$$
for $\forall j$. Therefore, we conclude that the Hörmander's condition holds for $X_t$. Then the ergodic proposition 8.1 of \citep{bellet2006ergodic} implies that the forward diffusion $X_s$ converges to the standard Gaussian distribution.
\end{proof}

\begin{remark}
A recent study \citep{dockhorn2022score} proposed to improve the diffusion model by enlarging the spatial space (where the generated samples lie in) to the "phase" space: $x \rightarrow (x,v)$. Then the corresponding joint forward diffusion $(x_t,v_t)$ satisfies the Critically-Damped Langevin diffusion:
\begin{equation} \label{eq:langevin_sde}
    \begin{pmatrix} dx_t \\ dv_t \end{pmatrix} = \begin{pmatrix} M^{-1} v_t \\ -x_t \end{pmatrix}dt +  \begin{pmatrix} \textbf{0}_d \\ -\Gamma M^{-1} v_t \end{pmatrix} dt +
    \begin{pmatrix} 0 \\ \sqrt{2\Gamma} \end{pmatrix}dW_t.
\end{equation}
If the coupling mass $M=1$, the drift part of Eq. \ref{eq:langevin_sde} can be decomposed to a symmetric part $R^{-1}$ and an \textbf{non-trivial} anti-symmetric part $\omega$ of (\ref{simple}) by setting:
$$R^{-1}:= \left(\begin{array}{@{}cc@{}}
  \begin{matrix}
  \ 0, 
  \end{matrix}
  & 0  \\

  \ 0,  &
  2 \Gamma I
\end{array}\right)\ \ ,\ \  \omega:= \left(\begin{array}{@{}cc@{}}
  \begin{matrix}
  \ 0, 
  \end{matrix}
  & -I  \\

  \ I,  &
  0
\end{array}\right).$$
It's straightforward to check that they rigorously fit the conditions of Corollary \ref{damped}. Therefore, we conclude from Corollary \ref{damped} that the Damped Langevin diffusion converges to the standard Gaussian distribution of the enlarged phase space $(x,v) \in \mathbb{R}^{2d}$, which coincides with the results of Appendix B.2 in \citep{dockhorn2022score}.  
\end{remark}

\subsection{Discussion of Section 3.2} \label{Section 3.2}
In this section, following the arguments from \citep{huang2021variational}, we demonstrate how to estimate the score gradient vector field $\nabla \log p(x)$ by the analytically tractable conditional score gradient vector field (conditioned on a previous time).

To prove (\ref{conditional}), by adapting Eq. 31 of \citep{huang2021variational}, it's enough to show that
$$\mathbb{E}_{X_t} [s^T_{\theta}(X_t , t) \cdot \nabla \log p_t(X_t)] = \mathbb{E}_{X_s, X_t}[s^T_{\theta}(X_t , t) \cdot \nabla \log p_t(X_t|X_s)].$$
Transforming the expectation to probabilistic integration, we have
\begin{align} \label{con}
\mathbb{E}_{X_t} [s^T_{\theta}(X_t , t) \cdot \nabla \log p_t(X_t)] & = \int p_t(x) s^T_{\theta}(x , t) \cdot \nabla \log p_t(x) dx\\
& = \int s^T_{\theta}(x , t) \int \nabla p_t(x|x_s) p_s(x_s) dx dx_s\\
& =\int \int p_s(x_s) p_t(x|x_s) \nabla p_t(x|x_s) dx dx_s\\
& = \mathbb{E}_{X_s, X_t}[s^T_{\theta}(X_t , t) \cdot \nabla \log p_t(X_t|X_s)],
\end{align}
for $0 \leq s < t$. By quadratic expanding $\mathbb{E}_{X_t}\norm{\mathbf{s}_{\theta}(X_t , t) - \nabla \log p_t(X_t)}^2$ and plugging in (\ref{con}), equality (\ref{conditional}) follows directly. 

To implement our discretized FP-diffusion forward diffusion, we usually choose $s = t-1$, the immediate time step before $t$. Then from $t-1$ to $t$, the conditional score gradient vector field of $p_t (x_t|x_{t-1})$ is the Gaussian score function, which is analytically tractable.

\subsection{Discussion of Section 3.3} \label{Section 3.3}
In this section, we prove Theorem \ref{simple version} by applying Ito's formula and martingale representation theorem.

Recall that the time-change of Eq. \ref{sum} satisfies
\begin{equation}
\label{eq:f_sde}
    dX_t = \beta'(t)(-\frac{1}{2}R^{-1} - \omega) X_t dt + \sqrt{\beta'(t)R^{-1}}dW_t,
\end{equation}
where $X_0$ is a fixed point. Let $Y_t : = e^{(\frac{1}{2}R^{-1} + \omega)\beta(t)} X_t$, then by Ito's formula,
\begin{equation} \label{Y}
 Y_t =\int_0^t e^{(\frac{1}{2}R^{-1} + \omega)\beta(s)} \sqrt{\beta'(s)R^{-1}}dW_s.
\end{equation}
From the martingale representation theorem, $Y_t$ is a Gaussian random variable for each $t$. Therefore, to fully determine the distribution of $X_t$, we only need to calculate the expectation and variance formulas of $X_t$. By the definition of stochastic integration, we have
$$\mathbf{E}[X(t)] =e^{(-\frac{1}{2}R^{-1} - \omega)\beta(t)}X_0. $$
Utilizing the Ito's isometry to (\ref{Y}), we get
$$Var[Y_t] =\int_0^t e^{\beta(s)(R^{-1} + 2\omega)} \beta'(s)R^{-1} ds.$$
Suppose $\omega = 0$, then
$$Var[X_t] =\mathbf{I} - e^{- \beta(t) R^{-1}},$$
where $\mathbf{I}$ denotes the identity matrix of $\mathbf{R}^d$.
Suppose $R^{-1} = \mathbf{I}$, since the Lie bracket $[\mathbf{I} + 2\omega, \mathbf{I} - 2\omega] =0$, we further obtain
$$Var[X_t] =\mathbf{I} - e^{- \beta(t) \mathbf{I}}. $$
In conclusion, we have proved Theorem \ref{simple version}.

\subsection{How to parameterize symmetric and anti-symmetric matrix} \label{parameterize}
To implement FP-Drift and FP-Noise models practically, we need to find an efficient way to parameterize positive-definite symmetric and anti-symmetric matrix.

Given a full-rank anti-symmetric matrix $B$, there always exist an orthogonal matrix $P$ such that






  






  $$B = P \text{diag}\begin{Bmatrix}\begin{bmatrix}

      0 & \lambda_1 \\

      -\lambda_1 & 0 

  \end{bmatrix}, \cdots \begin{bmatrix}

      0 & \lambda_n \\

      -\lambda_n & 0 

  \end{bmatrix}\end{Bmatrix} P^T,$$
  where $\{\lambda_1,\dots,\lambda_n\}$ are nonzero numbers.
  Then, the inverse of $\mathbf{I} + B$ (appeared in subsection \ref{parameterize}) is:
$$(B + I)^{-1} = P \text{diag}\begin{Bmatrix}\begin{bmatrix}

      \frac{1}{1+\lambda_1^{2}} & \frac{-\lambda_1}{1+\lambda_1^{2}} \\

      \frac{\lambda_1}{1+\lambda_1^{2}} & \frac{1}{1+\lambda_1^{2}}

  \end{bmatrix}, \cdots  \begin{bmatrix}

      \frac{1}{1+\lambda_n^{2}} & \frac{-\lambda_n}{1+\lambda_n^{2}} \\

      \frac{\lambda_n}{1+\lambda_n^{2}} & \frac{1}{1+\lambda_n^{2}}

  \end{bmatrix}\end{Bmatrix} P^T.$$
For positive-definite symmetric matrices, there always exist an orthogonal matrix $P$ such that
$$R = P \text{diag}\{\lambda_1, \dots, \lambda_n\} P^T,$$
where $\{\lambda_1,\dots,\lambda_n\}$ are positive numbers. 

To apply the above method, we only need to parameterize orthogonal matrices in an efficient and expressive way. By treating orthogonal matrices as elements in $SO(n)$ orthogonal group, we utilize the exponential map to parameterize orthogonal matrices $P$:
$$P = \exp{H}.$$
Note that $H$ is an element that belongs to the lie algebra $so(n)$, which can be generated by upper triangular matrices.
\begin{figure}[t]
    \centering
    \includegraphics[width=0.6\textwidth]{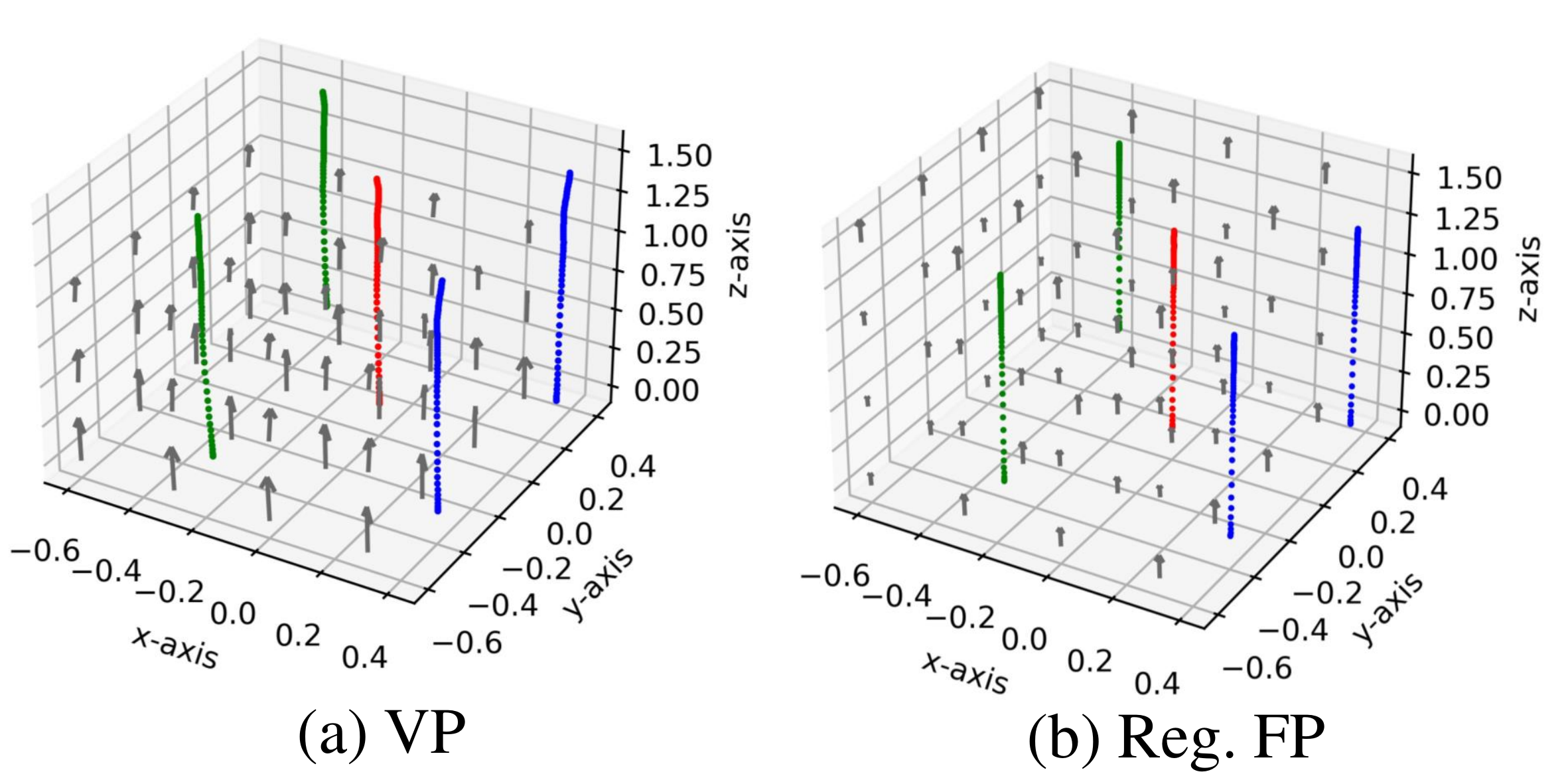}
    \caption{Integral trajectories of two SDEs}
    \label{fig:toy_trajectory}
    \vspace{-0.2cm}
\end{figure}

\section{Experiments}
\subsection{Learned FP SDEs from synthetic 3D examples}
Figure \ref{fig:toy_trajectory} plots four 3D integration trajectories of the probabilistic flows (with respect to the fixed VP and learned FP-Diffusion models) starting at random initial positions. It's obvious that the trajectories of our flexible model are more straight than the fixed VP model, which demonstrates the power of selecting more regular generating paths of our FP-Diffusion model.

\subsection{Image Generation} 
\paragraph{Implementation Details.}
Following \citep{ho2020denoising} and \citep{song2020score}, we rescale the range of the images into $[-1, 1]$ before inputting them into the model. In the FP-Diffusion model, $\beta(t)$ is an linearly increasing function with respect to the time $t$, i.e., $\beta(t)=\Bar{\beta}_{min}+t(\Bar{\beta}_{max}-\Bar{\beta}_{min})$ for $t \in [0, 1]$. It's worth mentioning that DDPM adopts a discretization form of this time scheduler, where $\beta_i = \frac{\Bar{\beta}_{min}}{N} + \frac{i-1}{N(N-1)}(\Bar{\beta}_{max}-\Bar{\beta}_{min})$. These two forms are actually equivalent when $N \rightarrow \infty$. For all experiments, we set $\Bar{\beta}_{max}$ as $20$ and $\Bar{\beta}_{min}$ as $0.1$, which are also used in \citep{ho2020denoising} and \citep{song2020score}. As discussed in \ref{parameterize}, we only need to parameterize the upper triangular matrices $H$ and the diagonal elements $\Lambda = diag\{\lambda_1, \cdots, \lambda_n\}$ in the FP-Drift and FP-Noise models. Particularly, both $H$ and $\Lambda$ are initialized with a multivariate normal distribution, and we adopt a exponential operation on $\Lambda$ to keep it a positive vector.
As described in Section 4.2, we leverage a U-net style neural network to fit the score function of the reverse-time diffusion process. We keep the model architecture and the parameters of the score networks consistent with previous SOTA diffusion models (e.g., \citep{song2020score}) for a fair comparison. All models are trained with the Adam optimizer with a learning rate $2 \times 10^{-4}$ and a batchsize $96$. 

In the MNIST experiment, we first train the whole model for $50$k iterations and train the score model for another $250$k iterations with our Mix training strategy. We report the NLL of the model based on the last checkpoint.
In the CIFAR10 experiment, the training iterations of both stage 1 and stage 2 are $600$k. We also report the FIDs and NLL of the model based on the last checkpoint.

\paragraph{Results.} We provide more random samples from our best FP-Drift model's checkpoint in Fig. \ref{fig:app_cifar}. We also provide the learned forward process of FP-Noise model in Fig. \ref{fig:app_cifar_forward}.

\begin{figure}
    \centering
    \includegraphics{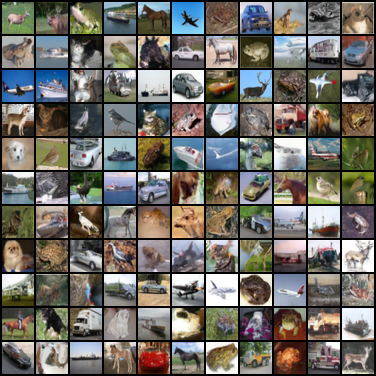}
    \caption{CIFAR-10 samples from FP-Drift}
    \label{fig:app_cifar}
\end{figure}

\begin{figure}
    \centering
    \includegraphics{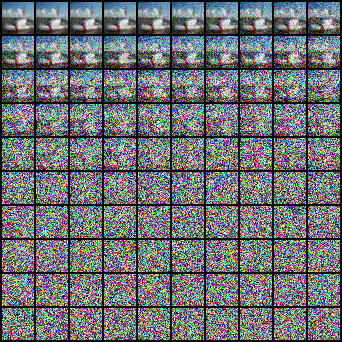}
    \caption{The learned forward process of FP-Noise on CIFAR-10}
    \label{fig:app_cifar_forward}
\end{figure}

\end{document}